\documentclass[11pt]{article}
\usepackage{tikz}
\tikzset{
  treenode/.style = {circle,draw, align=center,
                     top color=white, bottom color=blue!20},
  special/.style     = {treenode, font=\ttfamily\normalsize, bottom color=red!30},
  dummy/.style      = {treenode, font=\ttfamily\normalsize}
}
\usepackage{natbib}
\usepackage{hyperref}       
\usepackage{url}    
\usepackage{bbm}
\usepackage{booktabs}       
\usepackage{amsfonts}       
\usepackage{nicefrac}       
\usepackage{enumitem}
\usepackage{ltablex}
\usepackage{multirow}
\usepackage{mathrsfs}
\usepackage{times}
\usepackage{amsmath, amssymb, amsthm, graphicx, bm}
\usepackage[ruled,vlined,noresetcount]{algorithm2e}

\usepackage{mwe}
\usepackage[section]{placeins}
\usepackage{enumitem}
\usepackage{nicefrac} 
\usepackage{mathtools}
\usepackage{cleveref}
\usepackage{thmtools,thm-restate}
\newcommand{\E}{\mathbb{E}}
\newcommand{\R}{\ensuremath{\mathbb{R}}}

\newcommand{\alg}{\ensuremath{\mathcal{A}}}
\newtheoremstyle{definition}
  {}
  {}
  {\itshape}
  {}
  {\bfseries}
  {.}
  { }
  {\thmname{#1}\thmnumber{ #2}\thmnote{ (#3)}}

  \newtheoremstyle{theorem}
  {}
  {}
  {\itshape}
  {}
  {\bfseries}
  {.}
  { }
  {\thmname{#1}\thmnumber{ #2}\thmnote{ (#3)}}

\theoremstyle{theorem}
\newtheorem{theorem}{Theorem}[section]
\newtheorem{proposition}[theorem]{Proposition}
\theoremstyle{definition}

\newcommand{\Scal}{\ensuremath{\mathcal{S}}} 
 
\newcommand{\Acal}{\ensuremath{\mathcal{A}}}

\newcommand{\poly}{\textrm{poly}}
\newcommand{\dist}{\textrm{dist}}
\newcommand{\cnf}{\textsc{3-CNF}}
\newcommand{\sat}{\textsc{3-Sat}}
\newcommand{\ksat}{$k$\textsc{-Sat}}
\newcommand{\usat}{\textsc{Unique-3-Sat}}

\newcommand{\linear}[1]{\textsc{Linear-#1-RL}}
\newcommand{\klinear}{\textsc{Linear-}k\textsc{-RL}}
\newcommand{\reth}{\text{rETH}}
\newcommand{\npc}{\textsc{NP}}
\newcommand{\rpc}{\textsc{RP}}

\newcommand{\lp}{\left}
\newcommand{\rp}{\right}

\usepackage{fullpage}
\usepackage{xcolor}
\usepackage{algorithmic}
\usepackage[margin=1in]{geometry}
\usepackage{environ}
\usepackage{mdframed}
\newmdenv[
  topline=false,
  bottomline=false,
  rightline=false,
  linewidth=0.8pt,
  skipabove=\topsep,
  skipbelow=\topsep,
  innertopmargin=2pt,
  innerbottommargin=0pt
]{siderules}

\NewEnviron{problem}[1]{%
\begin{center}\fbox{\parbox{3in}{%
    {\centering\scshape #1\par}%
    \parskip=1ex
    \everypar{\hangindent=1em}%
    \BODY
}}\end{center}}

\title{Computational-Statistical Gaps in Reinforcement Learning}
\author{%
  Daniel Kane\thanks{Supported by NSF Award CCF-1553288 (CAREER) and a Sloan
  Research Fellowship.}\\
  University of California, San Diego\\
  \texttt{dakane@eng.ucsd.edu}
    \and
    Sihan Liu\\
    University of California, San Diego\\
    \texttt{sil046@ucsd.edu}
    \and
    Shachar Lovett\thanks{Supported by NSF Award CCF-1909634.}\\
  University of California, San Diego\\
  \texttt{slovett@cs.ucsd.edu}
    \and
    Gaurav Mahajan\\
    University of California, San Diego\\ 
    \texttt{gmahajan@eng.ucsd.edu}
}
\date{\today}
\makeatletter
\newcommand*{\rom}[1]{\expandafter\@slowromancap\romannumeral #1@}
\makeatother
\begin{document}
\maketitle

\newcommand{\gm}[1]{\textsf{\color{magenta}{GM: #1}}}
\newcommand{\shachar}[1]{[\textsf{\color{red}{SL: #1}}]}
\newcommand{\sihan}[1]{[\textsf{\color{red}{Sihan: #1}}]}
\newtheorem{fact}[theorem]{Fact}
\begin{abstract}
    Reinforcement learning with function approximation has recently achieved tremendous results in applications with large state spaces. This empirical success has motivated a growing body of theoretical work proposing necessary and sufficient conditions under which efficient reinforcement learning is possible. From this line of work, a remarkably simple minimal sufficient condition has emerged for sample efficient reinforcement learning: MDPs with optimal value function $V^*$ and $Q^*$ linear in some known low-dimensional features. In this setting, recent works have designed sample efficient algorithms which require a number of samples polynomial in the feature dimension and independent of the size of state space. They however leave finding computationally efficient algorithms as future work and this is considered a major open problem in the community.
    
    In this work, we make progress on this open problem by presenting the first computational lower bound for RL with linear function approximation: unless NP=RP, no randomized polynomial time algorithm exists for deterministic transition MDPs with a constant number of actions and linear optimal value functions. To prove this, we show a reduction from \textsc{Unique-Sat}, where we convert a CNF formula into an MDP with deterministic transitions, constant number of actions and low dimensional linear optimal value functions. This result also exhibits the first computational-statistical gap in reinforcement learning with linear function approximation, as the underlying statistical problem is information-theoretically solvable with a polynomial number of queries, but no computationally efficient algorithm exists unless NP=RP.
   Finally, we also prove a quasi-polynomial time lower bound under the Randomized Exponential Time Hypothesis.
\end{abstract}


\section{Introduction}
Function approximation has a long history in reinforcement learning \citep{tsitsiklis1996feature,Bertsekas2009,munos2008finite} and game playing \citep{shannon1950xxii,tesauro1995temporal}. More recently, this merger of reinforcement learning's algorithmic techniques with supervised learning's generalization schemes has achieved tremendous results in various applications with large state spaces, in areas such as game playing \citep{mnih2013playing,silver2017mastering,openai}, robotics \citep{kober2013reinforcement} and biology \citep{fold21}. Since, one would expect the statistical and computational demand for these algorithms to grow at least linearly with the size of the state space \citep{jaksch2010near}, it is quite surprising that these algorithms generalize so well in large state spaces. That said, the computational requirements for existing algorithms have become exceedingly high. For example, AlphaZero was trained on 5000 tensor processing units (TPUs) for 13 days \citep{Silver1140} and OpenAI Five trained its \textsc{Dota2} bots using 128000 CPUs \citep{openai} for 180 days (10 months in real time). This leads to a natural fundamental question: are such data and compute requirements fundamental or can we design efficient algorithms for these applications? More generally: \emph{what minimal properties of environments leads to efficient RL algorithms?}

Over the last decade, this question has driven a growing body of theoretical work showing when \emph{sample efficiency} is possible in RL for particular model classes, such as State
Aggregation \citep{lihong2009disaggregation,dong2020provably}, Linear
MDPs \citep{yang2019sample,jin2019provably}, Linear Mixture
MDPs \citep{modi2019sample,ayoub2020model,zhou2021provably}, Reactive
POMDPs \citep{krishnamurthy2016pac}, Block
MDPs \citep{du2019provably}, FLAMBE \citep{agarwal2020flambe},
Reactive PSRs \citep{littman2001predictive}, Linear Bellman
Complete \citep{munos2005error,zanette2020learning}. More generally, there are also a few lines of work which propose general frameworks, consisting of structural conditions which permit sample efficient RL; these include the Bellman rank \citep{jiang2016contextual}, Witness rank \citep{sun2018model}, Bilinear Classes \citep{bilinear2021} and Bellman Eluder \citep{jin2021bellman}. The goal in these latter works is to develop a unified theory of generalization in RL, analogous to the more classical notions in statistical complexity (e.g. VC-theory and Rademacher complexity) relevant for supervised learning. 

A surprisingly minimal assumption which arose from these works is Linear $Q^*\&V^*$ \citep{bilinear2021} where both optimal value function $V^*$ and optimal action-value function $Q^*$ are linear in some known low-dimensional features. \cite{bilinear2021} showed that in this setting, there exists sample efficient RL algorithms which \emph{regardless of the number of actions} require a number of samples polynomial in the feature dimension and independent of the size of the state space. However, when only either $V^*$ or $Q^*$ are linear, a series of works \citep{weisz2020exponential,wang2021exponential,weisz2021tensorplan,foster2021statistical} showed that a phase transition occurs as one increases the number of actions: sample efficient algorithms exist for constant number of actions, and quickly transform into information theoretic exponential lower bounds as the number of actions exceeds the dimension of the features underlying $Q^*$ or $V^*$. 

Even though we have made considerable progress in understanding the minimal assumptions from the statistical perspective, the computational aspect of this problem is largely unknown. All the settings mentioned above (except under strong assumptions like linear transitions \citep{jin2019provably} and deterministic rewards \citep{wen2013efficient}) do not have computationally efficient algorithms and previous works \citep{jiang2016contextual,bilinear2021,weisz21sl1} leave designing computationally efficient algorithms as an important open problem. On the other hand, in spite of failed search for such computationally efficient algorithms over the last few years, there are no computational lower bounds for any of these settings (although previous attempts \citep{low1} have shown inefficiency of specific algorithms). A case of particular interest is RL under linear function approximation with constant many actions, which includes linear $Q^*$ \citep{weisz2021tensorplan}, linear $V^*$ \citep{weisz2021queryefficient}, linear $Q^*\& V^*$ \citep{bilinear2021} and linear $Q^*\& V^*$ (reachable states) \citep{weisz2021tensorplan}. In all these settings, we have statistically efficient algorithms when the number of actions are $O(1)$, but all the algorithms \cite{du2020agnostic,bilinear2021,weisz2021tensorplan} take times either exponential in $d$ or $H$. Designing polynomial time algorithms for any of these settings is considered a major open problem in the community.

\subsection{Our Contributions} 
In this work, we present the first computational lower bounds for RL with linear function approximation. Before stating our main results, we first need to state some key definitions that we use throughout the paper.

\paragraph{Markov Decision Process (MDP).} We first define the framework for reinforcement learning, a Markov Decision Process (MDP). We define a deterministic MDP as a tuple $M = \left(\Scal,\Acal, R, P\right)$, where $\Scal$ is the state space, $\Acal$ is the action space, $R:\Scal\times\Acal\mapsto \Delta([0,1])$ is the stochastic reward function \footnote{$\Delta([0,1])$ denotes the set of all distributions over interval $[0,1]$.}, and $P:\Scal\times\Acal\mapsto \Scal$ is the deterministic transition function. An MDP $M$ defines a discrete time sequential decision process where the agent starts from a starting state $s_0\in \Scal$. Then, at each time $t$, the agent at some current state $S_t$, takes action $A_t$, receiving reward $R_t \sim R(S_t, A_t)$ and transitions to next state $S_{t+1}$. This goes on till the agent reaches the end state $\bot$. Each such trajectory/path from starting state $s_0$ to end state $\bot$ is of length at most horizon $H$.
A deterministic, stationary policy $\pi:\Scal \mapsto \Acal$ specifies a decision-making strategy in which the agent chooses actions adaptively based on the current state, i.e. $A_t = \pi(S_t)$.
Given a policy $\pi$ and a state-action pair $(s,a) \in \Scal \times \Acal$, the $Q$-function and $V$-function under a policy $\pi$ are defined as
\begin{equation}
    \label{eq:value}
    V^\pi(s)=\E\left[\sum_{t = 0}^{\tau-1}R(S_{t},A_{t})\mid S_0 =s, \pi\right],\quad Q^\pi(s,a) = \E\left[\sum_{t = 0}^{\tau - 1}R(S_{t},A_{t})\mid S_0 =s, A_0 = a, \pi\right],
\end{equation}
    where $S_{1}, A_{1}, \ldots S_{\tau-1}, A_{\tau - 1}$ are obtained by executing policy $\pi$ in the MDP $M$ and $\tau$ is the first time when policy $\pi$ reaches the end state $\bot$, that is $S_\tau = \bot$ where it always holds that $\tau \le H$. We use $Q^*$ and $V^*$ to denote the optimal value functions \[
        V^*(s) = \sup_{\pi} V^\pi(s)\, , \quad  Q^*(s,a) = \sup_{\pi} Q^\pi(s,a)\, , \quad s \in \Scal, a \in \Acal
    \] 
    We say that the optimal value functions $V^*$ and $Q^*$ can be written as a linear function of $d$-dimensional features $\psi \colon \Scal \cup (\Scal \times \Acal) \to \R^d$ if for all state $s$ and action $a$, $V^*(s) = \langle \theta, \psi(s)\rangle$ and $Q^*(s,a) = \langle \theta, \psi(s,a)\rangle$ for some fixed $\theta \in \R^d$ independent of $s$ and $a$.

    \paragraph{Computational Problems.} We next introduce \sat, a satisfiability problem for \cnf~formulas. In a \sat~problem, we are given as input, a \cnf~formula $\varphi$ with $v$ variables and $O(v)$ clauses and our goal is to decide if $\varphi$ is satisfiable. Our computational lower bound is based on a reduction from \usat, a variant of $3$-SAT. \usat~is the promise version of \sat~where the given formula is promised to have either $0$ or $1$ satisfying assignments.

The focus of this work is the computational RL problem, \klinear.  In a \klinear~problem with feature dimension $d$, we are given access to a deterministic MDP $M$ with $k$ actions and horizon $H = O(d)$ such that the optimal value functions $Q^*$ and $V^*$ can be written as a linear function of $d$-dimensional features $\psi$. Our goal is to output a good policy, which we define as any policy $\pi$ that satisfies $V^\pi > V^* - 1/4$. Note that here $V^\pi$ and $V^*$ refers to the value of the policy $\pi$ and optimal policy respectively at the starting state and is always in $[0,H]$ \footnote{in our constructions, we satisfy the more stringent condition that $V^* \in [0,1]$.}.
Moreover, the constant $1/4$ can be replaced by any arbitrary constant $< 1$. From now on, we always assume number of actions $k$ is 2 or 3.

\begin{center}
    \begin{minipage}{0.9\textwidth}
    \hrulefill\\[5pt]
    \textbf{Complexity problem}\quad \klinear\\[3pt]
    \begin{tabular}{@{}p{0.1\linewidth} @{}p{0.89\linewidth}}
         \textit{Oracle:} & a deterministic MDP $M$ with $k$ actions, optimal value functions $V^*$ and $Q^*$ linear in $d$ dimensional features $\psi$ and horizon $H = O(d)$.\\
         \textit{Goal:}&  find policy $\pi$ such that  $V^\pi > V^* - 1/4$.
    \end{tabular} \\[5pt]
    \vphantom{.}\hrulefill
    \end{minipage}
    \end{center}


We now describe how the algorithm interacts with the MDP. We assume that the algorithm has access to the associated (i) reward function $R$, (ii) transition function $P$ and (iii) features $\psi$. For all these functions, the algorithm provides a state $s$ and action $a$ (if needed) and receives a random sample from the distribution $R(s,a)$ (for the reward function), the state $P(s,a)$ (for the transition function) or feature $\psi(s)$ or $\psi(s,a)$ (for the features). We assume that each call accrues constant runtime and input/output for these functions are of size polynomial in feature dimension $d$.


We will often talk about randomized algorithm $A$ solving a problem in time $t$ with error probability $p$. By this we mean (i) $A$ runs in time $O(t)$; (ii) for satisfiability problems, it returns YES on positive input instances with probability at least $1-p$ and returns NO on negative input instances with probability $1$; and (iii) for RL problem, it returns a good policy with probability at least $1-p$. 

 \subsubsection{No polynomial time algorithm for \linear{2}}
 With these considerations in mind, we present our main result that asserts that unless \npc=\rpc, no randomized polynomial time algorithm can find a good policy in deterministic MDPs with a constant number of actions and linear optimal value functions. 
\begin{theorem}[\linear{2}~$\in$ \rpc $\implies$ \npc=\rpc]
    \label{thm:main}
    Unless \npc=\rpc, no randomized algorithm can solve \linear{2}~with feature dimension $d$ in time polynomial in $d$ with error probability $1/10$.
\end{theorem}

This resolves the open problem from \cite{weisz2021tensorplan} and \cite{bilinear2021} by showing that unless \rpc=\npc, no polynomial time randomized algorithm exists for deterministic transition MDPs with a constant number of actions and linear optimal value functions. 

Our main technical contribution is a reduction from \usat~to \linear{3}~such that a polynomial time algorithm for \linear{3}~implies a polynomial time algorithm for \usat. To achieve this, we use the input for \usat: a \cnf~formula $\varphi$ with $v$ variables, to design an input for \linear{3}: an MDP $M_{\varphi}$ with $3$ actions and optimal value functions $V^*$ and $Q^*$ linear in $d$-dimensional features. On a high level, the MDP is constructed such that each state represents an assignment to the \usat~variables and the goal is to ``search'' for the solution to the \usat~instance. In particular, at each state, the $3$ actions available to the agent correspond to an unsatisfied clause which ensures at least one action available to the agent decreases the distance to the solution. To incentivize finding the solution, a large reward is awarded on reaching the solution and a very small expected reward on reaching the horizon (this reward is small enough that any polynomial time RL algorithm only receives $0$ reward with high probability on reaching the horizon). This ensures that (i) finding a good policy also finds the satisfying assignment of $\varphi$ and (ii) the optimal value functions $V^*$ and $Q^*$ are linear in some low dimensional features. We present this construction in \Cref{sec:lower-3}. 

To get lower bounds for \linear{2}, we use the same construction as above with a small modification. We replace the choice of 3 actions $a_1$, $a_2$ and $a_3$ at every state with a depth-2 binary tree, where the first action is $a_1$ and the second action leads to a new state which has actions $a_2$ and $a_3$. This allows us to simulate the hard $3$-action MDP using a $2$-action MDP while increasing our feature dimension $d$ by at most a quadratic factor. We present this construction in \Cref{sec:lower2}. 

These reductions allow us to simulate a polynomial time algorithm for \usat~on input $\varphi$ by running the polynomial time algorithm for \linear{2}~on MDP $M_{\varphi}$. More formally, our reduction gives a polynomial relationship between the complexity of \usat~and \linear{2}: a polynomial $d^{q}$ time algorithm for \linear{2}~implies a polynomial $v^{O(q^2)}$ time algorithm for \usat. 

\begin{proposition}
    \label{prop:poly}
    Suppose $q\geq 1$. If \linear{2}~with feature dimension $d$ can be solved in time $d^{q}$ with error probability $1/10$, then \usat~with $v$ variables can be solved in time $v^{O(q^2)}$ with error probability $1/8$.
\end{proposition}

This relates the complexity of \usat~to \linear{2}~and \linear{3}. To relate these problems to complexity class \npc, we use a seminal result from \cite{vazval85} which showed that uniqueness of solution can not be used to solve search problems quickly. In particular, they showed a randomized polynomial time reduction from \sat~to \usat.
\begin{restatable}[Valiant-Vazirani Theorem]{theorem}{vvthm}
        \label{thm:vv}
        Unless \npc=\rpc, no polynomial time randomized algorithm can solve \usat~with error probability $1/8$.
\end{restatable} Combining our reduction with Valiant-Vazirani Theorem proves our main result--\Cref{thm:main}. 



\subsubsection{Quasi-Polynomial Lower Bound for \linear{2}}
We now present computational lower bound under a strengthening of \npc~$\neq$ \rpc~conjecture, Randomized Exponential Time Hypothesis (\reth) \citep{reth2014hardness}, which asserts that probabilistic algorithms can not decide if a given \sat~problem with $v$ variables and $O(v)$ clauses is satisfiable in sub-exponential time.

\begin{restatable}[Randomized Exponential Time Hypothesis (\reth)]{definition}{defreth}
    \label{def:reth}
	There is a constant $c > 0$ such that no \emph{randomized} algorithm can decide \sat~with $v$ variables in time $2^{c v}$ with error probability $1/2$.
\end{restatable}

Randomized Exponential Time Hypothesis along with many variants motivated by Exponential Time Hypothesis \citep{eth2001} have been influential in discovering hardness results for a variety of problems see, e.g. \cite{cygan2015lower,vassilevska2019hardness}. Under Randomized Exponential Time Hypothesis, our main result is a quasi-polynomial computational lower bound for learning good policies in deterministic MDPs with linear optimal value functions.
\begin{restatable}[Quasi-polynomial lower bound for \linear{2}]{theorem}{coraction}
    \label{cor:3hardness}
    Under \reth, no randomized algorithm can solve \linear{2}~with feature dimension $d$ in time $d^{O(\log d/ \log \log d)}$ with error probability $1/10$.
\end{restatable} 

This improves over our super-polynomial lower bound albeit depending on a much stronger hardness assumption. To prove this result, we use a different choice of parameters in our reduction and set the feature dimension $d$ to be sub-exponential in the number of variables $v$ to get the following:
\begin{proposition}
    \label{prop:exp}
    If \linear{2}~with feature dimension $d$ can be solved in time $d^{O(\log d/\log \log d)}$ with error probability $1/10$, then \usat~with $v$ variables can be solved in time $2^{O(v/\log v)}$ with error probability $1/8$.
\end{proposition}

Here its important to note that we can not use Valiant-Vazirani Theorem to relate \usat~and \sat, since it is consistent with Valiant-Vazirani Theorem that \usat~is solvable in $2^{\sqrt{v}}$ time but \sat~takes $2^{v}$ time. Therefore, we use a more refined lower bound for \usat~from \cite{calabro2008unique} which showed that if \usat~with $v$ variables can be solved in time $2^{\alpha v}$ for every $\alpha > 0$, then so can \ksat~for all $k\geq 3$. 

\begin{theorem}[\cite{calabro2008unique}]
    \label{thm:cc}
    Assuming \reth~is true, there exists a constant $c> 0$ such that no randomized algorithm can solve \usat~with $v$ variables in time $2^{c v}$ with error probability $1/2$.
\end{theorem}

In conjunction with our reduction, this gives a quasi-polynomial lower bound for \linear{2}~under \reth. We leave as an open problem if the techniques introduced in this work can be used to prove an exponential lower bound for \linear{2} under \reth. 


Our results give evidence that even though having linear optimal value functions is sufficient for sample efficient reinforcement learning \cite{bilinear2021}, it is not sufficient for computationally efficient reinforcement learning. More assumptions are required for computationally efficient algorithms, in addition to optimal value functions $Q^*$ and $V^*$ being linear in low-dimensional features, for example sub-optimality gap \citep{du2020agnostic}. We hope that this work will open up new research avenues for finding minimal sufficient conditions for computationally efficient reinforcement learning. We now discuss a few further notable implications of this work.

\begin{itemize}
    \item \emph{Computational-Statistical Gap}: There are many problems which exhibit computational-statistical gaps i.e.~regimes where the underlying statistical problem is information theoretically possible but no computationally efficient algorithm exists. Examples include community detection \citep{HOLLAND1983109,mcsherry2001,abbe2015detection}, planted clique \citep{noga98,barak2019} and sparse principal component analysis \citep{berthet13,berthet2013optimal}. To the best of our knowledge, our computational lower bound is the first computational-statistical gap in reinforcement learning with function approximation. When both optimal value functions $Q^*$ and $V^*$ are linear, MDPs with any number of actions are statistically easy to solve \citep{bilinear2021} but our results show that no polynomial time algorithm can solve these MDPs even with a constant number of actions, unless \npc=\rpc.
    
    \item \emph{Natural Problem in} \textsc{NP} $\setminus$ \textsc{P}: There has been quite a lot of recent work in complexity theory literature on proving quasi-polynomial lower bounds based on Exponential Time Hypothesis (for e.g. dense constraint satisfaction problems \citep{aaronson2014qp1}, approximating best nash equilibrium \citep{braverman2015qp2} and approximating densest $k$-subgraph with perfect completeness \citep{braverman2017qp3}). This work adds RL with deterministic transition, linear bounded optimal value functions $V^*$, $Q^*$ and constant number of actions as another natural problem in NP but not in P unless \npc=\rpc.
\end{itemize}

\paragraph{Remainder of this paper.} In \Cref{sec:lower-3} and \Cref{sec:lower2}, we present our lower bound constructions for $3$ action and $2$ action MDPs respectively.
\section{Lower Bound for MDPs with 3 actions}
\label{sec:lower-3}
In this section, we will prove the reduction, \Cref{prop:poly3} and \Cref{prop:exp3}, restated versions of \Cref{prop:poly} and \Cref{prop:exp} for \linear{3}. The overall idea is to first build a randomized algorithm $\alg_{SAT}$ which can decide \usat~using a randomized algorithm $\alg_{RL}$ which solves \linear{3}. 
The two reductions only differ in their settings of parameters. 

In the first setting, which we use to prove that no polynomial time algorithm exists for \linear{3}, we set the feature dimension $d$ to be polynomial in the number of variables $v$. Under this setting, we can build a polynomial time randomized algorithm for \usat~using a polynomial time randomized algorithm for \linear{3}.

\begin{proposition}
    \label{prop:poly3}
    Suppose $q\geq 1$. If \linear{3}~with feature dimension $d$ can be solved in time $d^{q}$ with error probability $1/10$, then \usat~with $v$ variables can be solved in time $O(v^{8q + 16q^2})$ with error probability $1/8$.
\end{proposition}

In the second setting, which we use to prove a quasi-polynomial lower bound for \linear{3}, we set the feature dimension $d$ to be sub-exponential in the number of variables $v$. This allows us to transform an exponential time lower bound for \usat~into a quasi-polynomial lower bound for \linear{3}.

\begin{restatable}{proposition}{propexp}
    \label{prop:exp3}
    If \linear{3}~with feature dimension $d$ can be solved in time $d^{\log d/(32 \log \log d)}$ with error probability $1/10$, then \usat~with $v$ variables can be solved in time $2^{O(v/\log v)}$ with error probability $1/8$.
\end{restatable}

Before we prove these results, we give a brief outline of our reduction from \usat~to \linear{3}. On a high level, we construct an MDP where the goal is to "search" for the solution $w^*$ to a~\usat~instance with $v$ variables. In particular, at each time, the agent is given an unsatisfied clause and asked to flip assignment for a variable present in the clause. Notice that since the clause is unsatisfied, there must be at least one variable whose assignment differs from the solution and therefore, the agent can ``reach'' the solution in at most $d(w,w^*)$ steps. To incentivize the agent, if the agents at time $l$ finds the solution i.e. $w = w^*$ or reaches the end of the MDP i.e. $l = H$, it receives reward according to the following degree-$r$ polynomial
\begin{align*}
    g(l, w) = \left(1 - \frac{l + \dist(w, w^*)}{H + v}\right)^{r}.
\end{align*}
We show how to build an MDPs from a \usat~instance in \Cref{sub:const}. Furthermore, we show that the optimal value functions $V^*$ and $Q^*$ for the constructed MDP are linear in $d = O(v^r)$-dimensional features. Since the expected reward at last layer of the MDP is $O(v^{-r^2})$ (which can be replaced with $0$ for any $\poly(d)$ time RL algorithm), the only non-zero reward is achieved by solving the underlying \usat~instance, proving our reduction. We give a formal argument in \Cref{sub:alg}, where we show how to build a randomized algorithm for \usat~using a randomized algorithm for \linear{3}. In \Cref{sub:setting}, we discuss the two different settings of parameters which will prove \Cref{prop:poly3} and \Cref{prop:exp3}.

\begin{figure}
    \centering
    \begin{tikzpicture}
        [
          grow                    = right,
          sibling distance        = 3em,
          level distance          = 6em,
          edge from parent/.style = {draw, -latex},
          every node/.style       = {font=\footnotesize}
        ]
        \node [special, label={[align=center]left:$x_1 \lor x_2 \lor x_3$\\ $(-1,-1,-1,-1)$ \\ $l=0$}] {}
          child { node [dummy, label={[align=center]right:$x_1 \lor x_2 \lor \bar x_3$\\$(-1,-1,1,-1)$\\$l=1$}] {}
            edge from parent node [below] {$x_3$} }
            child { node [dummy] {}
            edge from parent node [above] {$x_2$} }
          child { node [special, label={[align=center, xshift=-1em]above:$\bar x_1 \lor x_2 \lor x_3$\\ $(1,-1,-1,-1)$ \\$l=1$}] {}
          child { node [dummy, label={[align=center]right:$\bar x_1 \lor x_2 \lor \bar x_3$\\$(1,-1,1,-1)$\\$l=2$}] {}
            edge from parent node [below] {$x_3$} }
            child { node [dummy] {}
            edge from parent node [above] {$x_1$}}
            child { node [special, label={[align=center, xshift=-1em]above:$\bar x_1 \lor x_3 \lor x_4$\\ $(1,1,-1,-1)$ \\ $l=2$}] {}
              child { node [dummy] {}
                edge from parent node [below] {$x_1$} }
              child { node [dummy] {}
                edge from parent node [above] {$x_3$}}
              child { node [special, label={[align=center]above:$(1,1,-1,1)$ \\$l=3$}] {}
              child { node [special] {$\bot$}
              edge from parent node [above] {} }
                      edge from parent node [above] {$x_4$} }
              edge from parent node [above] {$x_2$} }
                    edge from parent node [above] {$x_1$} };
      \end{tikzpicture}
\caption{Example construction of $3$-action MDP $M_\varphi$ from a \cnf~formula $(x_1 \lor x_2 \lor x_3) \land (\bar x_1 \lor x_2 \lor  x_3) \land ( \bar x_1 \lor x_3 \lor x_4 ) \land (x_1 \lor x_2 \lor \bar x_3) \land (\bar x_1 \lor x_2 \lor \bar x_3) \land (\bar x_3 \lor \bar x_3 \lor \bar x_3)\land (x_1 \lor x_1 \lor x_1)$. The only satisfying assignment for this formula is $(1,1,-1,1)$. The states are labelled by the corresponding assignment and unsatisfied clause which decides the available actions. The states in the optimal path are colored in red.}
\label{fig:lower3}
\end{figure}
\subsection{From \cnf~formulas to 3-action MDPs}
\label{sub:const}
We will start by defining a mapping from an input of \usat~problem: \cnf~formula $\varphi$ with $v$ variables and $O(v)$ clauses to an MDP $M_{\varphi}$ with $3$ actions and $H = O(d)$ horizon with optimal value functions linear in $d$ dimensions. Our informal goal is to design an MDP $M_{\varphi}$ such that finding a good policy also implies finding the satisfying assignment for the formula $\varphi$. We now formally describe the MDP $M_\varphi$ when the formula $\varphi$ has a unique satisfying assignment $w^* \in \{-1,1\}^v$ and later show how the MDP $M_\varphi$ differs when the formula $\varphi$ has no solution. See \Cref{fig:lower3} for an example.

\paragraph{Transitions.} In our setting, it will be useful to visualize an MDP as a tree, where nodes represent states and edges represent actions. A policy is then a sequence of actions or equivalently a path in the aforementioned tree. The MDP $M_\varphi$ is a ternary tree i.e. each state/node in the tree has $3$ children. The transitions/dynamics are deterministic i.e. the first action goes to first child, the second action goes to second child and so on.

\paragraph{Assignments.} Each state is associated with an assignment to the $v$ variables i.e. a binary vector in $\{-1, 1\}^v$ and a natural number $l$ denoting the depth of the state. Our goal here is to choose assignments such that it is always possible to choose an action which decreases the hamming distance to the satisfying assignment. The root in the tree is associated with the all zeroes assignment $(-1, -1, \ldots, -1)$. For any state $s$ with a non-satisfying assignment $w = (w_1, w_2, \ldots, w_v) \neq w^*$, the assignment associated to the three children are as follows. Since $w$ is not a satisfying assignment, consider the first unsatisfied clause with variables $x_{i_1}, x_{i_2}, x_{i_3}$. The first child is associated with the assignment where the $i_1$-th bit of $w$ is flipped, the second child is associated with vector where $i_2$-th bit is flipped and so on. More formally, the assignment associated to $j$-th child is $(w'_1, w'_2, \ldots, w'_v)$ where $w'_k = \neg w_k$ if $k = i_j$ and $w'_k = w_k$ otherwise. The two exceptions to this are (i) states with the satisfying assignment $w^*$ and (ii) states at the last level $H$. For such states, all actions go to the end state $\bot$.

\paragraph{Rewards.} To ensure that finding good policies implies finding the satisfying assignment in our MDP, we will only give rewards when a satisfying assignment is found or at the last layer. More formally, the rewards everywhere are zero except on (i) states with the satisfying assignment $w^*$ and (ii) states on the last level $H$. In both the cases above, say the state is at level $l$ with assignment $w$, then the associated reward distribution for any action is a Bernoulli distribution $Ber(g(l, w))$ where 
\begin{align*}
    g(l,w) = \left(1 - \frac{l + \dist(w, w^*)}{H + v}\right)^{r}
\end{align*} 
and the Bernoulli distribution $Ber(p)$ is $1$ with probability $p$ and $0$ with probability $1-p$. Here $r$ is a parameter which we will specify in \Cref{sub:setting}. When the formula $\varphi$ has no satisfying assignment, all rewards are $0$.
Note that in our simulation (\Cref{sub:alg}), we don't know/use $w^*$ and instead use an approximate reward function that is easy to compute.

\paragraph{Linear Optimal Value Functions.} We next show that in the MDP $M_\varphi$, the optimal value functions $V^*$ and $Q^*$ can be written as a linear function of $d=O(v^r)$ dimensional features $\psi$, where $\psi(s)$ or $\psi(s,a)$ depends only on $w$, the corresponding assignment, and $l$, the depth of the state.

\begin{proposition}
\label{prop:linear}
For any state $s$ in level $l$ with assignment $w$ and action $a$,
\begin{enumerate}[label=(\roman*)]
    \item the optimal value function is $V^*(s) = g(l,w)$.
    \item for large enough $v$, there exists features $\psi(s), \psi(s,a) \in \R^d$ with feature dimension $d \le 2 v^r$ depending only on state $s$ and action $a$; and $\theta \in \R^d$ depending only on $w^*$ such that $V^*$ and $Q^*$ can be written as a linear function of features $\psi$ i.e. $V^*(s) = \langle \theta, \psi(s)\rangle$ and $Q^*(s,a) = \langle \theta, \psi(s,a)\rangle$.
\end{enumerate}
\end{proposition}
\begin{proof}
    To prove our first claim, we start by showing that there exists a policy $\pi$ that achieves this value for each state. Let $\pi$ be the policy which for any state $s$ with assignment $w\neq w^*$ chooses the action which decreases the hamming distance $\dist(w,w^*)$ by $1$. Note that one such action always exists in our construction, since a satisfying assignment satisfies all clauses. Therefore, from a state $s$ at level $l$ with assignment $w$, we can reach a state with assignment $w_1$ such that either (i) $w_1$ is a satisfying assignment or (ii) $w_1$ is at the last level and on the optimal path from $w$ to $w^*$ i.e. $\dist(w, w^*) = \dist(w, w_1) + \dist(w_1, w^*)$. In both cases, \[
        V^\pi(s) = \left(1 - \frac{l + \dist(w, w_1) + \dist(w_1, w^*)}{H + v}\right)^{r} = g(l,w)
    \] 
    Next, for any other policy $\pi'$ that ends on state $s'$ at level $l'$ with assignment $w'$ (i.e. either $l' = H$ or $w' = w^*$), we have \[
        V^{\pi'}(s) = \left(1 - \frac{l'+ \dist(w', w^*)}{H + v}\right)^{r} \leq \left(1 - \frac{l+ \dist(w, w') + \dist(w', w^*)}{H + v}\right)^{r} \leq g(l,w)
    \] where the first inequality follows from $l' - l \geq \dist(w, w')$. This proves our first claim about $V^*$ i.e. $V^*(s) = g(l,w)$.
    
     To prove our second claim, that $V^*$ and $Q^*$ can be written as a linear function of features $\psi$, we will show that $V^*(s)$ can be written as a polynomial of degree at most $r$ in $w^*$. To see why this is enough, we set $\theta$ to be all monomials in $w^*$ of degree at most $r$. That is, each coordinate of $\theta$ corresponds to a multiset $S \subset [v]$ of size $|S| \le r$, and its value is $\theta_S = \prod_{i \in S}w^*_i$.  
     We set $\psi(s)$ to be the corresponding coefficients in the polynomial $V^*$. Then, we can write $V^*(s) = \langle \theta, \psi(s)\rangle$. 
     Since, there are at most $\sum_{i=0}^r v^i \leq 2 v^r$ many coefficients we can set the feature dimension as $d = 2 v^r$.

    Finally, we prove that $V^*(s)$ can be written as a polynomial of degree at most $r$ in $w$ and $w^*$. Firstly hamming distance $\dist(w,w^*)$ is linear in both $w$ and $w^*$ i.e. \[
        \dist(w, w^*) = \frac{v - \langle w, w^*\rangle}{2}
    \] Our claim follows from noting that $g(l,w)$ is a polynomial of degree $r$ in $\dist(w, w^*)$. Note that linear $V^*$ implies linear $Q^*$ in deterministic MDPs for $\psi(s,a) = \psi(P(s,a))$, since by definition, in MDPs with deterministic transition, $Q^*(s,a) = V^*(P(s,a))$.
\end{proof}

Even though $\psi(s)$ does not depend on $w^*$, unlike the constructions of \cite{weisz21sl2,weisz2021tensorplan,wang2021exponential}, $\psi(s)$ does depend on the MDP $M_\varphi$ making this construction statistically easy but computationally hard to solve. 

\subsection{From RL algorithms to 3-SAT algorithms} \label{sub:alg}

We now build a randomized algorithm $\alg_{SAT}$ for \usat~using a randomized algorithm $\alg_{RL}$ for the RL problem. However, as mentioned before, since the runtime for $\alg_{RL}$ accrues only constant runtime for each call to the MDP oracle, to efficiently build $\alg_{SAT}$ using $\alg_{RL}$, we need to be able to efficiently simulate the calls to MDP oracle, namely: calls to the reward function, the transition function and the features. To do so, we build an ``approximate'' simulator $\bar M_{\varphi}$ for the MDP oracle $M_\varphi$. The simulator $\bar M_{\varphi}$ is exactly MDP $M_\varphi$ in terms of transition function and features associated with the MDP $M_\varphi$, but differs in the reward function at the last layer which is always $0$ for the simulator $\bar M_{\varphi}$. This modification is crucial for an efficient reduction because unlike transitions and features for any state which can be computed in time $\text{poly}(d)$ on the MDP $M_\varphi$, the rewards at the last layer when $\dist(w,w^*) \neq 0$ require access to $w^*$ which can not be done efficiently.  With the purposed modification, we can execute each call to simulator $\bar M_{\varphi}$ in time $\text{poly}(d)$.

\paragraph{Algorithm.} On input \cnf~formula $\varphi$, $\alg_{SAT}$ runs the algorithm $\alg_{RL}$ replacing each call to MDP oracle $M_{\varphi}$ with the corresponding call to simulator $\bar M_{\varphi}$. Recall that the output for the RL algorithm in our setting is a sequence of actions. If the sequence of actions returned by $\alg_{RL}$ ends on a state with assignment $w$, $\alg_{SAT}$ outputs YES if $w$ is the satisfying assignment and returns NO otherwise.

\paragraph{Correctness.} We set the horizon $H = v^r$. We will assume throughout that $r\geq 2$ and that the runtime of $\alg_{RL}$ is $\leq v^{r^2/4}$. Different settings of $r$ satisfying these assumptions will prove \Cref{prop:poly3} and \Cref{prop:exp3} for $3$-action MDPs, which we will discuss in \Cref{sub:setting}. To complete our reduction, we will show the following:  \begin{enumerate}
    \item[(i)] If algorithm $\alg_{RL}$ outputs a policy $\pi$ such that $V^\pi > V^* - 1/4$, then $\alg_{SAT}$ on \cnf~formula $\varphi$ outputs YES if $\varphi$ is satisfiable and NO otherwise.
    \item[(ii)] If $\alg_{RL}$ with access to MDP oracle $M_\varphi$ outputs a policy $\pi$ such that $V^\pi > V^* - 1/4$ with error probability $1/10$, then $\alg_{RL}$ with access to simulator $\bar M_\varphi$ outputs a policy $\pi$ such that $V^{\pi} > V^* - 1/4$ with error probability $1/8$.
\end{enumerate}
These together will show that $\alg_{SAT}$ solves \usat~with error probability $\leq 1/8$. We start by proving that if $\alg_{RL}$ succeeds on MDP $\bar M_\varphi$, then $\alg_{SAT}$ succeeds on \cnf~formula $\varphi$. This follows from the fact that any good policy in the MDP $M_\varphi$ must reach a state with satisfying assignment $w^*$.

\begin{proposition} \label{prop:rl-sat}
    Suppose $r > 1$ and horizon $H = v^r$. If $\alg_{RL}$  outputs a policy $\pi$ such that $V^\pi > V^* - 1/4$, then $\alg_{SAT}$ on \cnf~formula $\varphi$ outputs YES if $\varphi$ is satisfiable and NO otherwise.
\end{proposition}
\begin{proof}
    Since algorithm $\alg_{SAT}$ always returns NO on an unsatisfiable formula, we restrict our attention to a satisfiable formula $\varphi$.  In the MDP $M_\varphi$, (i) rewards are ``very small'' everywhere except on reaching the satisfying assignment i.e. the expected reward at the last layer in the MDP $M_\varphi$ is upper bounded by (for large enough $v$ and $r >1$) \begin{align*}
        \left(1 - \frac{H}{H + v}\right)^r = \left(\frac{v}{H + v}\right)^r \leq v^{-r^2 + r} < 1/4
    \end{align*} and (ii) the optimal value $V^*$ is large \begin{align*}
        V^* \geq \left(1 - \frac{v}{H + v}\right)^r = \left( 1 + \frac{v}{v^{r}}\right)^{-r} \geq  1- \frac{rv}{v^r} \geq \frac{1}{2}
    \end{align*} where the second last inequality follows from Bernoulli's inequality and the last inequality holds for large enough $v$ and $r>1$. Therefore, if the value of policy is large i.e. $V^{\pi} > V^* - 1/4$, then the policy $\pi$ (and therefore the corresponding sequence of actions) has to end on a state with the satisfying assignment $w^*$. By construction of $\alg_{SAT}$, this implies $\alg_{SAT}$ will succeed on the formula $\varphi$.
\end{proof}

Since we can not simulate the rewards on MDP oracle $M_\varphi$ efficiently, our reduction runs the algorithm $\alg_{RL}$ on an approximate simulator $\bar M_\varphi$. However, it's not clear why $\alg_{RL}$ would still succeed when each call to MDP oracle is replaced by a call to the simulator $\bar M_\varphi$. The following proposition shows that in fact  $\alg_{RL}$ would succeed on the outputs of simulator $\bar M_\varphi$ albeit with a smaller constant probability.
\begin{proposition}\label{prop:remove-bar}
    Suppose $r \geq 2$ and horizon $H = v^r$. Suppose $\alg_{RL}$ with access to MDP oracle $M_\varphi$ runs in time $v^{r^2/4}$ and outputs a policy $\pi$ such that $V^\pi > V^*  - 1/4$ with error probability $1/10$. Then $\alg_{RL}$ with access to simulator $\bar M_\varphi$, still running in time $v^{r^2/4}$, outputs a policy $\pi$ such that $V^{\pi} > V^* - 1/4$ with error probability $1/8$.
\end{proposition}
\begin{proof}
    Let $\Pr_{M_\varphi}$ and $\Pr_{\bar M_\varphi}$ denote the distribution on the observed rewards and output policies induced by the algorithm $\alg_{RL}$ when running on access to MDP oracle $M_\varphi$ and simulator $\bar M_\varphi$ respectively. Let $R_i$ denote the reward received on the last layer at the end of $i$-th trajectory. Let $T$ be the total number of trajectories sampled by algorithm $\alg_{RL}$ when running on access to MDP oracle $M_\varphi$. By our assumption, $\alg_{RL}$ runs in time $v^{r^2/4}$ and therefore $T\leq v^{r^2/4}$. Since the expected reward at the last layer in the MDP $M_\varphi$ is upper bounded by (for large enough $v$ and $r \geq 2$) \begin{align*}
        \left(1 - \frac{H}{H + v}\right)^r = \left(\frac{v}{H + v}\right)^r \leq v^{-r^2 + r} \leq v^{-\frac{r^2}{2}}
    \end{align*} and  and the algorithm only visits at most $v^{r^2/4}$ states on last layer, we get by the union bound that with high probability all the rewards at the last level are zero. More precisely (and assuming $v$ is large enough),
    \begin{align*}
        \Pr_{M_\varphi}\left[R_i = 0 ~\forall i \in [T]\right] \geq 1 - v^{-r^2/4} \ge \frac{4}{5}
    \end{align*} 
    We say $\alg_{RL}$ succeeds with access to $M_\varphi$ (or $\bar M_\varphi$) if the output policy $\pi$ after running for time at most $v^{r^2/4}$ satisfies $V^\pi > V^* - 1/4$. Using the above reasoning and the assumption that $\alg_{RL}$ succeeds with access to MDP oracle $M_\varphi$ with probability $9/10$ implies 
    \begin{align*}
    \Pr_{M_\varphi}\left[\alg_{RL}~\text{succeeds with access to}~M_\varphi \mid R_i = 0 ~\forall i \in [T]\right] &\geq \frac{\frac{9}{10}-\frac{1}{5}}{\frac{4}{5}} =  
    \frac{7}{8}
    \end{align*} Note that the marginal distributions $\Pr_{M_\varphi}$ and $\Pr_{\bar M_\varphi}$ on output policy $\pi$ given $R_i = 0~\forall i \in [T]$ are exactly the same because MDP oracle $\bar M_\varphi$ and simulator $M_\varphi$ only differ on last layer rewards. This implies 
    \begin{align*}
        &\Pr_{\bar M_\varphi}\left[\alg_{RL}~\text{succeeds with access to}~\bar M_{\varphi} \mid R_i = 0 ~\forall i \in [T]\right] \\
        &= \Pr_{M_\varphi}\left[\alg_{RL}~\text{succeeds with access to}~M_{\varphi} \mid R_i = 0 ~\forall i \in [T]\right]
        \end{align*}
    Since, $\Pr_{\bar M_\varphi}\left[R_i = 0 ~\forall i \in [T]\right] = 1$, we conclude that
    \begin{align*}
    \Pr_{\bar M_\varphi}\left[\alg_{RL}~\text{succeeds with access to}~\bar M_{\varphi}\right] &\geq  \frac{7}{8}
    \end{align*} 
\end{proof}

\subsection{Setting of Parameters} 
\label{sub:setting}
It follows from \Cref{prop:linear,prop:rl-sat,prop:remove-bar} that if \linear{3}~with feature dimension $d = 2 v^r$ can be solved in time $v^{r^2/4}$ with error probability $1/10$, then \usat~with $v$ variables can be solved in time $d \cdot v^{r^2/4}$ with error probability $1/8$ (here the extra $d$ factor is because each call to the simulator $\bar M_\varphi$ takes $d$ time). In this section, we discuss the two different settings of $r$ we use to prove our lower bounds. As we increase $r$, we decrease the expected reward available to the algorithm at the last layer on the order of $v^{-O(r^2)}$, making the problem harder. However, increasing $r$ also increases the feature dimension on the order of $v^{r}$. This non-polynomial gap in the feature dimension and expected reward at the last layer will give our main reduction. 


In the first setting, we will set $r$ to be a constant wrt number of variables $v$ and prove that a polynomial algorithm for \linear{3}~implies a polynomial algorithm for \usat. 



\begin{proof}[Proof of \Cref{prop:poly3}]
For any $q\geq 1$, we set \begin{equation}
    \label{eq:rv}
    r = 8q\, .
\end{equation} 
Note that $q\geq 1$ implies $r\geq 2$. Therefore, to prove our proposition, we just need to show \begin{align}
    \label{eq:to1}d^q &\leq v^{r^2/4}\\
    \label{eq:to2}d \cdot v^{r^2/4} &\leq v^{8q + 16q^2 + 1}
\end{align}
 under this setting of $d$ and $r$. Here the first equation bounds the time complexity of \linear{3}~in terms of feature dimension $d$ and the second equation bounds the time complexity of \usat~in terms of the number of variables $v$.  \Cref{eq:to1} is true as
    \begin{align*}
        v^{\frac{r^2}{4}} = (v^{r})^{\frac{r}{4}} \geq d^{\frac{r}{8}} = d^{q}
    \end{align*} where the first inequality follows from $d \leq v^{2r}$ for large enough $v$ and the last equality follows from \Cref{eq:rv} above. \Cref{eq:to2} holds since
    \begin{align*}
        d \cdot v^{r^2/4} =  2 v^{r + r^2/4} = O(v^{8q + 16q^2}),
    \end{align*} where the first equality follows from $d = 2 v^r$ and the last equality follows from \Cref{eq:rv} for large enough $v$.
\end{proof}

In \Cref{sec:app-lower}, we prove a more general version, \Cref{lemma:3hardnessapp}, which shows that a quasi-polynomial algorithm for \linear{3}~implies a quasi-polynomial algorithm for \usat. 

In the second setting, we set $r^2$ to be almost linear in the number of variables $v$. This will prove \Cref{prop:exp3} for $3$-action MDPs. 

\begin{proof}[Proof of \Cref{prop:exp3}]
    This follows exactly as proof of \Cref{prop:poly3}. We set \[r = \left\lceil\frac{\sqrt{v}}{\log v}~\right\rceil\, .\] We proceed to show that (i) the time complexity of~\linear{3}~can be bounded by $v^{r^2/4}$ (ii) $d \cdot v^{r^2/4}$, which is the time complexity of~\usat~if (i) holds, can be bounded by $2^{O(v/\log v)}$.

With this setting, the time complexity of \linear{3}~simplifies to \begin{align*}
        d^{\frac{\log d}{32 \log \log d}} \leq v^{\frac{2r\log d}{32 \log \log d}} \leq v^{\frac{4r^2\log v}{32\log (r \log v)}} \leq v^{\frac{ 8r^2\log v}{32\log v}} = v^{\frac{r^2}{4}}
    \end{align*} where the first and second inequality follows from $v^r \leq d \leq v^{2r}$ and third inequality follows from our setting of $r$.    

    Similarly, the time complexity of~\usat~simplifies to\begin{align*}
        d \cdot v^{r^2/4} =  2 v^{r + \frac{r^2}{4}} \leq v^{r^2} \leq v^{\frac{4v}{\log^2 v}} = 2^{\frac{4v}{\log v}}
    \end{align*}
    where the first equality follows from $d = 2 v^r$, first inequality follows for $r \geq 2$ and large enough $v$ and the second inequality follows from our setting of $r$ above. 
\end{proof}
\section{Lower Bound for MDPs with 2 actions}
\label{sec:lower2}
In this section, we prove computational lower bound for \linear{2}. Similar to \Cref{sec:lower-3}, our proof is based on reduction from \usat.
We will modify the MDP $M_{\varphi}$ with three actions into $ M_{\varphi}$ by introducing some intermediate states. See Figure \ref{fig:2action} for an example of this modification for a single state.

\paragraph{Intermediate states.} Recall that in $M_{\varphi}$ each state is associated with an assignment. Let the $i$-th clause, which consists of three variables $x_{i_1}, x_{i_2}, x_{i_3}$, be the first unsatisfied clause. Then, the three actions available each correspond to flipping one of the variable in the clause. We will replace them by two actions: while one action still flips the last variable $x_{i_3}$, the other action leads to an intermediate state $s_{[i_1, i_2]}$. At the state $s_{[i_1, i_2]}$, two actions are available: one flips $x_{i_1}$ and the other flips $x_{i_2}$.

\paragraph{Depth of state.} In the 3 action MDP $M_{\varphi}$, the depth of a state is simply the length of the path that ends at the state. Here, we define the depth to be the number of non-intermediate states included in the path. That being said, the intermediate states will have the same depth as their parents.

\paragraph{Rewards.} The rewards are the same as those in the 3 action MDP. Namely, rewards are only given at last layer or when the assignment is $w^*$. In particular, for a state with assignment $w$ and depth $l$, the reward distribution is $Ber(g(l, w))$ where 
\begin{align*}
    g(l,w) = \left(1 - \frac{l + \dist(w, w^*)}{H + v}\right)^{r}.
\end{align*}

\begin{figure} \label{fig:2action}
    \centering
    \begin{tikzpicture}
        [
          grow                    = right,
          sibling distance        = 3em,
          level distance          = 6em,
          edge from parent/.style = {draw, -latex},
          every node/.style       = {font=\footnotesize}
        ]
        \node [special, label={[align=center]left:$x_1 \lor x_2 \lor x_3$\\ $(-1,-1,-1, \cdots) $ \\ $l=1$}] {}
          child { node [special, yshift=-1em, xshift=6em, label={[align=center, xshift=0em]right:$(-1,-1,1, \cdots)$ \\ $l=2$}] {}
            edge from parent node [below] {$x_3$} }
          child { node [dummy, label={[align=center, yshift=0.5em]above: $(-1,-1,-1,\cdots)$ \\ $l=1$}] {}
            child { node [special, label={[align=center, xshift=0em]right: $(-1,1,-1,\cdots)$ \\ $l=2$}] {}
              edge from parent node [above] {$x_2$} }
            child { node [special, label={[align=center, xshift=0em]right: $(1,-1,-1,\cdots)$ \\ $l=2$}] {}
              edge from parent node [above] {$x_1$} }
                    edge from parent node [above] {$[x_1,x_2]$} };
      \end{tikzpicture}
\caption{Part of a $2$-action MDP corresponding to the CNF clause $(x_1 \lor x_2 \lor x_3)$. The non-intermediate states are colored red and the intermediate states are colored blue.}
\end{figure}
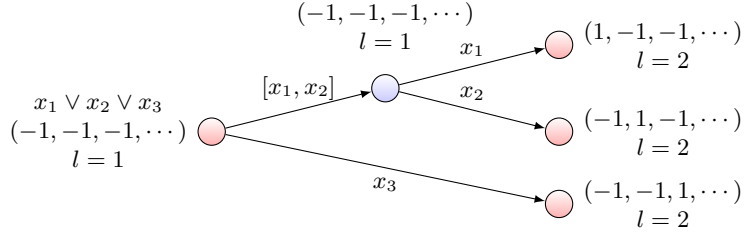

We now show that, even with this modification, the optimal value functions $V^*$ and $Q^*$ can still be written as a linear function of some low dimensional features.
\begin{proposition}
\label{prop:linear2}
For any state $s$ in level $l$ with assignment $w$ and action $a$,
\begin{enumerate}[label=(\roman*)]
    \item If $s$ is a non-intermediate state, then the optimal value function is $V^*(s) = g(l,w)$.
    \item If $s$ is an intermediate state that leads to actions which flip coordinates $i_1$ and $i_2$, then the optimal value function is
\begin{align*}
    V^*(s_{[i_1, i_2]}) = \left(1 - \frac{l + \dist(w, w^*) + 2 \cdot \mathbbm{1}\{ w_{i_1} = w^*_{i_1} \} \cdot \mathbbm{1}\{ w_{i_2} = w^*_{i_2} \}}{H + v}\right)^{r}.
\end{align*}
    \item for feature dimension $d = 2 v^{2r}$, there exists features $\psi(s), \psi(s,a) \in \R^d$ depending only on state $s$ and action $a$; and $\theta \in \R^d$ depending only on $w^*$ such that $V^*$ and $Q^*$ can be written as a linear function of features $\psi$ i.e. $V^*(s) = \langle \theta, \psi(s)\rangle$ and $Q^*(s,a) = \langle \theta, \psi(s,a)\rangle$.
\end{enumerate}
\end{proposition}
\begin{proof}
The proof for the value function of non-intermediate state is identical to that in the 3-action MDP.  
We proceed to argue the second claim. For an intermediate state, the value function will be identical to its parent if the two actions available include a wrong bit that ought to be flipped in the optimal assignment $w^*$. Otherwise, no matter what action the agent takes, it will reach a non-intermediate state whose
depth is $l+1$ and hamming distance is $\dist(w, w^*) + 1$. Compared to the value function of its parent, such intermediate state will have an extra $2$ in the numerator. We encode the situation with the indicator term $\mathbbm{1}\{ w_{i_1} = w^*_{i_1} \} \cdot \mathbbm{1}\{ w_{i_2} = w^*_{i_2} \}$. This then gives the value function for these intermediate states.

Lastly, like in the proof of \Cref{prop:linear}, it suffices to argue the value function is a degree $2r$ polynomial in $w$ and $w^*$. This is by noticing that (i) $\dist(w, w^*)$ is linear in $w$ and $w^*$; and (ii) $\mathbbm{1}\{ w_{i_1} = w^*_{i_1} \} \cdot \mathbbm{1}\{ w_{i_2} = w^*_{i_2} \}$ is quadratic in $w$ and $w^*$ i.e.
$$
\mathbbm{1}\{ w_{i_1} = w^*_{i_1} \} \cdot \mathbbm{1}\{ w_{i_2} = w^*_{i_2} \}
= \frac{1}{4} \cdot \lp(1 - w_{i_1} \cdot w^*_{i_1} \rp) \lp( 1 - w_{i_2} \cdot w^*_{i_2} \rp).
$$
Thus, the value function is overall a polynomial of degree $2r$ in $w^*$. As in \Cref{prop:linear}, we can set $\theta$ to be all monomials in $w^*$ of degree at most $2r$ and $\phi(s)$ to be the corresponding coefficients. Since there are at most $2 v^{2r}$ such monomials, this concludes the proof.
\end{proof}

By \Cref{prop:linear2}, the feature dimension $d$ of the 2 action MDP and the number of variables $v$ in the~\usat~instance are related by $d = 2 v^{2r} \leq v^{3r}$. We are now ready to prove \Cref{prop:poly} and \Cref{prop:exp}.

\begin{proof}[Proof of \Cref{prop:poly}]
Fix $q \geq 1$, we set $r = 12q$. Under this setting, we have
    \begin{align*}
        d^{q}
        \leq d^{ r/12 }
        \leq \lp( v^{3r} \rp)^{r/12}
        = v^{r^2/4} \, ,
    \end{align*}
    where the first inequality follows from the setting of $r$ and the second inequality follows from $d \leq v^{3r}$.
    The reduction then allows us to upper bound the complexity of~\usat~by
        $$
        d \cdot v^{r^2/4} 
        \leq v^{r^2/4+3r}
        = v^{O(q^2)} \, ,
        $$
    where the inequality again follows from $d \leq v^{3r}$.
\end{proof}

\begin{proof}[Proof of \Cref{prop:exp}]
    The proof follows similarly as proof of \Cref{prop:exp3}. The only difference is that since $d$ is now bounded by $v^{3r}$ instead of $v^{2r}$, we need the runtime of \linear{2}~in the assumption to also have a different constant in the exponent i.e. $d^{ \log d / \lp( 72 \log \log d \rp) }$.
\end{proof}



\section*{Acknowledgements}
The authors would like to thank Sham Kakade, Akshay Krishnamurthy, Ayush Sekhari, Wen Sun, Csaba Szepesvari and Gellert Weisz for enlightening discussions and comments on initial draft. 
\bibliographystyle{plainnat}
\bibliography{main.bib}

\begin{thebibliography}{55}
\providecommand{\natexlab}[1]{#1}
\providecommand{\url}[1]{\texttt{#1}}
\expandafter\ifx\csname urlstyle\endcsname\relax
  \providecommand{\doi}[1]{doi: #1}\else
  \providecommand{\doi}{doi: \begingroup \urlstyle{rm}\Url}\fi

\bibitem[Aaronson et~al.(2014)Aaronson, Impagliazzo, and
  Moshkovitz]{aaronson2014qp1}
Scott Aaronson, Russell Impagliazzo, and Dana Moshkovitz.
\newblock Am with multiple merlins.
\newblock \emph{2014 IEEE 29th Conference on Computational Complexity (CCC)},
  pages 44--55, 2014.

\bibitem[Abbe and Sandon(2015)]{abbe2015detection}
Emmanuel Abbe and Colin Sandon.
\newblock Detection in the stochastic block model with multiple clusters: proof
  of the achievability conjectures, acyclic bp, and the information-computation
  gap.
\newblock \emph{arXiv preprint arXiv:1512.09080}, 2015.

\bibitem[Agarwal et~al.(2020)Agarwal, Kakade, Krishnamurthy, and
  Sun]{agarwal2020flambe}
Alekh Agarwal, Sham Kakade, Akshay Krishnamurthy, and Wen Sun.
\newblock Flambe: Structural complexity and representation learning of low rank
  mdps.
\newblock \emph{arXiv preprint arXiv:2006.10814}, 2020.

\bibitem[Alon et~al.(1998)Alon, Krivelevich, and Sudakov]{noga98}
Noga Alon, Michael Krivelevich, and Benny Sudakov.
\newblock Finding a large hidden clique in a random graph.
\newblock In \emph{Proceedings of the Ninth Annual ACM-SIAM Symposium on
  Discrete Algorithms}, page 594–598, 1998.

\bibitem[Ayoub et~al.(2020)Ayoub, Jia, Szepesvari, Wang, and
  Yang]{ayoub2020model}
Alex Ayoub, Zeyu Jia, Csaba Szepesvari, Mengdi Wang, and Lin~F Yang.
\newblock Model-based reinforcement learning with value-targeted regression.
\newblock \emph{arXiv:2006.01107}, 2020.

\bibitem[Barak et~al.(2019)Barak, Hopkins, Kelner, Kothari, Moitra, and
  Potechin]{barak2019}
Boaz Barak, Samuel~B. Hopkins, Jonathan~A. Kelner, Pravesh Kothari, Ankur
  Moitra, and Aaron Potechin.
\newblock A nearly tight sum-of-squares lower bound for the planted clique
  problem.
\newblock \emph{SIAM J. Comput.}, 48:\penalty0 687--735, 2019.

\bibitem[Berner et~al.(2019)Berner, Brockman, Chan, Cheung, Debiak, Dennison,
  Farhi, Fischer, Hashme, Hesse, J{\'{o}}zefowicz, Gray, Olsson, Pachocki,
  Petrov, de~Oliveira~Pinto, Raiman, Salimans, Schlatter, Schneider, Sidor,
  Sutskever, Tang, Wolski, and Zhang]{openai}
Christopher Berner, Greg Brockman, Brooke Chan, Vicki Cheung, Przemyslaw
  Debiak, Christy Dennison, David Farhi, Quirin Fischer, Shariq Hashme,
  Christopher Hesse, Rafal J{\'{o}}zefowicz, Scott Gray, Catherine Olsson,
  Jakub Pachocki, Michael Petrov, Henrique~Pond{\'{e}} de~Oliveira~Pinto,
  Jonathan Raiman, Tim Salimans, Jeremy Schlatter, Jonas Schneider, Szymon
  Sidor, Ilya Sutskever, Jie Tang, Filip Wolski, and Susan Zhang.
\newblock Dota 2 with large scale deep reinforcement learning.
\newblock \emph{CoRR}, 2019.

\bibitem[Berthet and Rigollet(2013{\natexlab{a}})]{berthet13}
Quentin Berthet and Philippe Rigollet.
\newblock Complexity theoretic lower bounds for sparse principal component
  detection.
\newblock In Shai Shalev-Shwartz and Ingo Steinwart, editors, \emph{Proceedings
  of the 26th Annual Conference on Learning Theory}, volume~30, pages
  1046--1066, 2013{\natexlab{a}}.

\bibitem[Berthet and Rigollet(2013{\natexlab{b}})]{berthet2013optimal}
Quentin Berthet and Philippe Rigollet.
\newblock Optimal detection of sparse principal components in high dimension.
\newblock \emph{The Annals of Statistics}, 41\penalty0 (4):\penalty0
  1780--1815, 2013{\natexlab{b}}.

\bibitem[Bertsekas(2009)]{Bertsekas2009}
Dimitri~P. Bertsekas.
\newblock \emph{Neuro-Dynamic Programming}, pages 2555--2560.
\newblock Springer US, Boston, MA, 2009.

\bibitem[Braverman et~al.(2015)Braverman, Ko, and Weinstein]{braverman2015qp2}
Mark Braverman, Young~Kun Ko, and Omri Weinstein.
\newblock Approximating the best nash equilibrium in $n^{o(\log n)}$-time
  breaks the exponential time hypothesis.
\newblock In \emph{Proceedings of the Twenty-Sixth Annual ACM-SIAM Symposium on
  Discrete Algorithms}, page 970–982, 2015.

\bibitem[Braverman et~al.(2017)Braverman, Kun-Ko, Rubinstein, and
  Weinstein]{braverman2017qp3}
Mark Braverman, Young Kun-Ko, Aviad Rubinstein, and Omri Weinstein.
\newblock Eth hardness for densest-k-subgraph with perfect completeness.
\newblock SODA '17, page 1326–1341, USA, 2017. Society for Industrial and
  Applied Mathematics.

\bibitem[Calabro et~al.(2008)Calabro, Impagliazzo, Kabanets, and
  Paturi]{calabro2008unique}
Chris Calabro, Russell Impagliazzo, Valentine Kabanets, and Ramamohan Paturi.
\newblock The complexity of unique k-sat: An isolation lemma for k-cnfs.
\newblock \emph{Journal of Computer and System Sciences}, 74\penalty0
  (3):\penalty0 386--393, 2008.
\newblock Computational Complexity 2003.

\bibitem[Cygan et~al.(2015)Cygan, Fomin, Kowalik, Lokshtanov, Marx, Pilipczuk,
  Pilipczuk, and Saurabh]{cygan2015lower}
Marek Cygan, Fedor Fomin, Lukasz Kowalik, Daniel Lokshtanov, Dániel Marx,
  Marcin Pilipczuk, Michał Pilipczuk, and Saket Saurabh.
\newblock \emph{Lower Bounds Based on the Exponential-Time Hypothesis}, pages
  467--521.
\newblock 07 2015.
\newblock ISBN 978-3-319-21274-6.
\newblock \doi{10.1007/978-3-319-21275-3_14}.

\bibitem[Dann et~al.(2018)Dann, Jiang, Krishnamurthy, Agarwal, Langford, and
  Schapire]{low1}
Christoph Dann, Nan Jiang, Akshay Krishnamurthy, Alekh Agarwal, John Langford,
  and Robert~E. Schapire.
\newblock On oracle-efficient pac rl with rich observations.
\newblock In \emph{Proceedings of the 32nd International Conference on Neural
  Information Processing Systems}, page 1429–1439, 2018.

\bibitem[Dell et~al.(2014)Dell, Husfeldt, Marx, Taslaman, and
  Wahl\'{e}n]{reth2014hardness}
Holger Dell, Thore Husfeldt, D\'{a}niel Marx, Nina Taslaman, and Martin
  Wahl\'{e}n.
\newblock Exponential time complexity of the permanent and the tutte
  polynomial.
\newblock \emph{ACM Trans. Algorithms}, 10\penalty0 (4), 2014.

\bibitem[Dong et~al.(2020)Dong, Roy, and Zhou]{dong2020provably}
Shi Dong, Benjamin~Van Roy, and Zhengyuan Zhou.
\newblock Provably efficient reinforcement learning with aggregated states,
  2020.

\bibitem[Du et~al.(2021)Du, Kakade, Lee, Lovett, Mahajan, Sun, and
  Wang]{bilinear2021}
Simon Du, Sham Kakade, Jason Lee, Shachar Lovett, Gaurav Mahajan, Wen Sun, and
  Ruosong Wang.
\newblock Bilinear classes: A structural framework for provable generalization
  in rl.
\newblock In \emph{Proceedings of the 38th International Conference on Machine
  Learning}, volume 139, pages 2826--2836, 18--24 Jul 2021.

\bibitem[Du et~al.(2019)Du, Krishnamurthy, Jiang, Agarwal, Dud{\'\i}k, and
  Langford]{du2019provably}
Simon~S Du, Akshay Krishnamurthy, Nan Jiang, Alekh Agarwal, Miroslav
  Dud{\'\i}k, and John Langford.
\newblock Provably efficient {RL} with rich observations via latent state
  decoding.
\newblock In \emph{International Conference on Machine Learning}, 2019.

\bibitem[Du et~al.(2020)Du, Lee, Mahajan, and Wang]{du2020agnostic}
Simon~S Du, Jason~D Lee, Gaurav Mahajan, and Ruosong Wang.
\newblock Agnostic {Q}-learning with function approximation in deterministic
  systems: Tight bounds on approximation error and sample complexity.
\newblock In \emph{Advances in Neural Information Processing Systems}, 2020.

\bibitem[Foster et~al.(2021)Foster, Kakade, Qian, and
  Rakhlin]{foster2021statistical}
Dylan~J. Foster, Sham~M. Kakade, Jian Qian, and Alexander Rakhlin.
\newblock The statistical complexity of interactive decision making, 2021.

\bibitem[Holland et~al.(1983)Holland, Laskey, and Leinhardt]{HOLLAND1983109}
Paul~W. Holland, Kathryn~Blackmond Laskey, and Samuel Leinhardt.
\newblock Stochastic blockmodels: First steps.
\newblock \emph{Social Networks}, 5\penalty0 (2):\penalty0 109--137, 1983.

\bibitem[Impagliazzo and Paturi(2001)]{eth2001}
Russell Impagliazzo and Ramamohan Paturi.
\newblock On the complexity of k-sat.
\newblock \emph{J. Comput. Syst. Sci.}, 62\penalty0 (2):\penalty0 367–375,
  2001.

\bibitem[Jaksch et~al.(2010)Jaksch, Ortner, and Auer]{jaksch2010near}
Thomas Jaksch, Ronald Ortner, and Peter Auer.
\newblock Near-optimal regret bounds for reinforcement learning.
\newblock \emph{Journal of Machine Learning Research}, 11\penalty0 (4), 2010.

\bibitem[Jiang et~al.(2016)Jiang, Krishnamurthy, Agarwal, Langford, and
  Schapire]{jiang2016contextual}
Nan Jiang, Akshay Krishnamurthy, Alekh Agarwal, John Langford, and Robert~E.
  Schapire.
\newblock Contextual decision processes with low bellman rank are
  pac-learnable, 2016.

\bibitem[Jin et~al.(2020)Jin, Yang, Wang, and Jordan]{jin2019provably}
Chi Jin, Zhuoran Yang, Zhaoran Wang, and Michael~I Jordan.
\newblock Provably efficient reinforcement learning with linear function
  approximation.
\newblock In \emph{Conference on Learning Theory}, 2020.

\bibitem[Jin et~al.(2021)Jin, Liu, and Miryoosefi]{jin2021bellman}
Chi Jin, Qinghua Liu, and Sobhan Miryoosefi.
\newblock Bellman eluder dimension: New rich classes of rl problems, and
  sample-efficient algorithms.
\newblock \emph{arXiv preprint arXiv:2102.00815}, 2021.

\bibitem[Kober et~al.(2013)Kober, Bagnell, and Peters]{kober2013reinforcement}
Jens Kober, J~Andrew Bagnell, and Jan Peters.
\newblock Reinforcement learning in robotics: A survey.
\newblock \emph{The International Journal of Robotics Research}, 32\penalty0
  (11):\penalty0 1238--1274, 2013.

\bibitem[Krishnamurthy et~al.(2016)Krishnamurthy, Agarwal, and
  Langford]{krishnamurthy2016pac}
Akshay Krishnamurthy, Alekh Agarwal, and John Langford.
\newblock Pac reinforcement learning with rich observations.
\newblock In \emph{Proceedings of the 30th International Conference on Neural
  Information Processing Systems}, pages 1848--1856, 2016.

\bibitem[Li(2009)]{lihong2009disaggregation}
Lihong Li.
\newblock \emph{A Unifying Framework for Computational Reinforcement Learning
  Theory}.
\newblock PhD thesis, USA, 2009.
\newblock AAI3386797.

\bibitem[Littman et~al.(2001)Littman, Sutton, and Singh]{littman2001predictive}
Michael~L Littman, Richard~S Sutton, and Satinder~P Singh.
\newblock Predictive representations of state.
\newblock In \emph{NIPS}, volume~14, page~30, 2001.

\bibitem[McSherry(2001)]{mcsherry2001}
F.~McSherry.
\newblock Spectral partitioning of random graphs.
\newblock In \emph{Proceedings 42nd IEEE Symposium on Foundations of Computer
  Science}, pages 529--537, 2001.

\bibitem[Mnih et~al.(2013)Mnih, Kavukcuoglu, Silver, Graves, Antonoglou,
  Wierstra, and Riedmiller]{mnih2013playing}
Volodymyr Mnih, Koray Kavukcuoglu, David Silver, Alex Graves, Ioannis
  Antonoglou, Daan Wierstra, and Martin Riedmiller.
\newblock Playing atari with deep reinforcement learning.
\newblock \emph{arXiv preprint arXiv:1312.5602}, 2013.

\bibitem[Modi et~al.(2020)Modi, Jiang, Tewari, and Singh]{modi2019sample}
Aditya Modi, Nan Jiang, Ambuj Tewari, and Satinder Singh.
\newblock Sample complexity of reinforcement learning using linearly combined
  model ensembles.
\newblock In \emph{Conference on Artificial Intelligence and Statistics}, 2020.

\bibitem[Munos(2005)]{munos2005error}
R{\'e}mi Munos.
\newblock Error bounds for approximate value iteration.
\newblock In \emph{Proceedings of the National Conference on Artificial
  Intelligence}, volume~20, page 1006. Menlo Park, CA; Cambridge, MA; London;
  AAAI Press; MIT Press; 1999, 2005.

\bibitem[Munos and Szepesv{\'a}ri(2008)]{munos2008finite}
R{\'e}mi Munos and Csaba Szepesv{\'a}ri.
\newblock Finite-time bounds for fitted value iteration.
\newblock \emph{Journal of Machine Learning Research}, 9\penalty0 (5), 2008.

\bibitem[Senior et~al.(2020)Senior, Evans, Jumper, Kirkpatrick, Sifre, Green,
  Qin, Žídek, Nelson, Bridgland, Penedones, Petersen, Simonyan, Crossan,
  Kohli, Jones, Silver, Kavukcuoglu, and Hassabis]{fold21}
Andrew~W. Senior, Richard Evans, John Jumper, James Kirkpatrick, Laurent Sifre,
  Tim Green, Chongli Qin, Augustin Žídek, Alexander W.~R. Nelson, Alex
  Bridgland, Hugo Penedones, Stig Petersen, Karen Simonyan, Steve Crossan,
  Pushmeet Kohli, David~T. Jones, David Silver, Koray Kavukcuoglu, and Demis
  Hassabis.
\newblock Improved protein structure prediction using potentials from deep
  learning.
\newblock \emph{Nature}, 2020.

\bibitem[Shannon(1950)]{shannon1950xxii}
Claude~E Shannon.
\newblock Xxii. programming a computer for playing chess.
\newblock \emph{The London, Edinburgh, and Dublin Philosophical Magazine and
  Journal of Science}, 41\penalty0 (314):\penalty0 256--275, 1950.

\bibitem[Silver et~al.(2017)Silver, Schrittwieser, Simonyan, Antonoglou, Huang,
  Guez, Hubert, Baker, Lai, Bolton, et~al.]{silver2017mastering}
David Silver, Julian Schrittwieser, Karen Simonyan, Ioannis Antonoglou, Aja
  Huang, Arthur Guez, Thomas Hubert, Lucas Baker, Matthew Lai, Adrian Bolton,
  et~al.
\newblock Mastering the game of go without human knowledge.
\newblock \emph{nature}, 550\penalty0 (7676):\penalty0 354--359, 2017.

\bibitem[Silver et~al.(2018)Silver, Hubert, Schrittwieser, Antonoglou, Lai,
  Guez, Lanctot, Sifre, Kumaran, Graepel, Lillicrap, Simonyan, and
  Hassabis]{Silver1140}
David Silver, Thomas Hubert, Julian Schrittwieser, Ioannis Antonoglou, Matthew
  Lai, Arthur Guez, Marc Lanctot, Laurent Sifre, Dharshan Kumaran, Thore
  Graepel, Timothy Lillicrap, Karen Simonyan, and Demis Hassabis.
\newblock A general reinforcement learning algorithm that masters chess, shogi,
  and go through self-play.
\newblock \emph{Science}, 2018.

\bibitem[Sun et~al.(2019)Sun, Jiang, Krishnamurthy, Agarwal, and
  Langford]{sun2018model}
Wen Sun, Nan Jiang, Akshay Krishnamurthy, Alekh Agarwal, and John Langford.
\newblock Model-based {RL} in contextual decision processes: {PAC} bounds and
  exponential improvements over model-free approaches.
\newblock In \emph{Conference on Learning Theory}, 2019.

\bibitem[Tesauro et~al.(1995)]{tesauro1995temporal}
Gerald Tesauro et~al.
\newblock Temporal difference learning and td-gammon.
\newblock \emph{Communications of the ACM}, 38\penalty0 (3):\penalty0 58--68,
  1995.

\bibitem[Tsitsiklis and Van~Roy(1996)]{tsitsiklis1996feature}
John~N Tsitsiklis and Benjamin Van~Roy.
\newblock Feature-based methods for large scale dynamic programming.
\newblock \emph{Machine Learning}, 22\penalty0 (1):\penalty0 59--94, 1996.

\bibitem[Valiant and Vazirani(1985)]{vazval85}
L~G Valiant and V~V Vazirani.
\newblock Np is as easy as detecting unique solutions.
\newblock In \emph{Proceedings of the Seventeenth Annual ACM Symposium on
  Theory of Computing}, page 458–463, 1985.

\bibitem[Wang et~al.(2021)Wang, Wang, and Kakade]{wang2021exponential}
Yuanhao Wang, Ruosong Wang, and Sham~M. Kakade.
\newblock An exponential lower bound for linearly-realizable mdps with constant
  suboptimality gap, 2021.

\bibitem[Weisz et~al.(2020)Weisz, Amortila, and
  Szepesvári]{weisz2020exponential}
Gellert Weisz, Philip Amortila, and Csaba Szepesvári.
\newblock Exponential lower bounds for planning in mdps with
  linearly-realizable optimal action-value functions, 2020.

\bibitem[Weisz et~al.(2021{\natexlab{a}})Weisz, Amortila, Janzer,
  Abbasi-Yadkori, Jiang, and Szepesvári]{weisz21sl1}
Gell{\'e}rt Weisz, Philip Amortila, Barnabás Janzer, Yasin Abbasi-Yadkori, Nan
  Jiang, and Csaba Szepesvári.
\newblock On query-efficient planning in mdps under linear realizability of the
  optimal state-value function, 2021{\natexlab{a}}.

\bibitem[Weisz et~al.(2021{\natexlab{b}})Weisz, Amortila, and
  {Sz}epesv{\'a}ri]{weisz21sl2}
Gell{\'e}rt Weisz, Philip Amortila, and {Cs}aba {Sz}epesv{\'a}ri.
\newblock Exponential lower bounds for planning in mdps with
  linearly-realizable optimal action-value functions.
\newblock In \emph{Proceedings of the 32nd International Conference on
  Algorithmic Learning Theory}, volume 132, pages 1237--1264. PMLR,
  2021{\natexlab{b}}.

\bibitem[Weisz et~al.(2021{\natexlab{c}})Weisz, Szepesv{\'a}ri, and
  Gy{\"o}rgy]{weisz2021tensorplan}
Gell{\'e}rt Weisz, Csaba Szepesv{\'a}ri, and Andr{\'a}s Gy{\"o}rgy.
\newblock Tensorplan and the few actions lower bound for planning in mdps under
  linear realizability of optimal value functions.
\newblock \emph{arXiv preprint arXiv:2110.02195}, 2021{\natexlab{c}}.

\bibitem[Weisz et~al.(2021{\natexlab{d}})Weisz, Amortila, Janzer,
  Abbasi-Yadkori, Jiang, and Szepesvári]{weisz2021queryefficient}
Gellért Weisz, Philip Amortila, Barnabás Janzer, Yasin Abbasi-Yadkori, Nan
  Jiang, and Csaba Szepesvári.
\newblock On query-efficient planning in mdps under linear realizability of the
  optimal state-value function, 2021{\natexlab{d}}.

\bibitem[Wen and Van~Roy(2013)]{wen2013efficient}
Zheng Wen and Benjamin Van~Roy.
\newblock Efficient exploration and value function generalization in
  deterministic systems.
\newblock In \emph{Advances in Neural Information Processing Systems}, 2013.

\bibitem[Williams(2019)]{vassilevska2019hardness}
Virginia~Vassilevska Williams.
\newblock On some fine-grained questions in algorithms and complexity.
\newblock \emph{Proceedings of the International Congress of Mathematicians
  (ICM 2018)}, 2019.

\bibitem[Yang and Wang(2019)]{yang2019sample}
Lin Yang and Mengdi Wang.
\newblock Sample-optimal parametric {Q}-learning using linearly additive
  features.
\newblock In \emph{International Conference on Machine Learning}, 2019.

\bibitem[Zanette et~al.(2020)Zanette, Lazaric, Kochenderfer, and
  Brunskill]{zanette2020learning}
Andrea Zanette, Alessandro Lazaric, Mykel Kochenderfer, and Emma Brunskill.
\newblock Learning near optimal policies with low inherent bellman error, 2020.

\bibitem[Zhou et~al.(2021)Zhou, He, and Gu]{zhou2021provably}
Dongruo Zhou, Jiafan He, and Quanquan Gu.
\newblock Provably efficient reinforcement learning for discounted mdps with
  feature mapping.
\newblock In \emph{International Conference on Machine Learning}, pages
  12793--12802. PMLR, 2021.

\end{thebibliography}
\newpage
\appendix
\section{General Reduction from \usat~to \linear{3}}
\label{sec:app-lower}
\begin{restatable}{proposition}{thmaction}
    \label{lemma:3hardnessapp}
    Suppose $m \geq 0$ and $q\geq 1$. If \linear{3}~with feature dimension $d$ can be solved in time $d^{q \cdot (\log d)^{m/(m + 2)}}$ with error probability $1/10$, then \usat~with $v$ variables can be solved in time $v^{O(16^{m+2} q^{m+2}) \cdot (\log v)^m}$ with error probability $1/8$.
\end{restatable}
\begin{proof}
It follows from \Cref{prop:linear}, \Cref{prop:rl-sat} and \Cref{prop:remove-bar} that for $2 \leq r < v$, if \linear{3}~with feature dimension $d = 2 v^r$ can be solved in time $v^{r^2/4}$ with error probability $1/10$, then \usat~with $v$ variables can be solved in time $d \cdot v^{r^2/4}$ with error probability $1/8$ (as each call to simulator $\bar M_\varphi$ takes $d$ time). For any $m \geq 0$ and $q\geq 1$, we set \begin{equation}
    \label{eq:rv-appendix}
    r = \left\lceil \sqrt{(16q)^{m+2} \log^m v}~\right\rceil\, .
\end{equation} Note that $m \geq 0$ and $q\geq 1$ implies $r\geq 2$. Therefore, to prove our claim, we just need to show the following equations hold for our setting of $d$ and $r$: \begin{align}\label{eq:toprove1-appendix}
    v^{r^2/4} &\geq d^{q \cdot \log^{\frac{m}{m+2}}d}\\
    d \cdot v^{r^2/4} &= v^{O((16q)^{m+2}) \log^m v} \label{eq:toprove2-appendix}
\end{align} Here the first equation bounds the time complexity of \linear{3}~in terms of feature dimension $d$ and the second equation bounds the time complexity of \usat~in terms of the number of variables $v$.
\paragraph{Proof of \Cref{eq:toprove1-appendix}:}To prove the first inequality, we lower bound $r$ in terms of feature dimension $d$ as \begin{align} \label{eq:r-exp}
        r \geq 8 q \lp( \log d\rp)^{m/(m+2)}.
    \end{align} which can be proved by lower bounding $r^{m+2}$ as follows\begin{align*}
        r^{m+2} = r^2 \cdot r^m &\geq (16q)^{m+2} \log^m v \cdot r^m\\
         &= (16q)^{m+2} \log^m (v^r) \geq (16q)^{m+2} \log^m (\sqrt{d}) \geq (8q)^{m+2} \log^m d
    \end{align*} where the first inequality follows from our setting of $r$ and the second inequality follows from $d \leq 2 v^r \leq v^{2r}$ for $r > 0$ and large enough $v$. Substituting the lower bound in $v^{r^2/4}$, we can write the time complexity of \linear{3}~in terms of feature dimension $d$ as \begin{align*}
        v^{\frac{r^2}{4}} = (v^{r})^{\frac{r}{4}} \geq d^{\frac{r}{8}} \geq d^{q \cdot \log^{\frac{m}{m+2}}d}
    \end{align*} where the first inequality follows again from $d \leq v^{2r}$ and the second inequality follows from \Cref{eq:r-exp} above.
    \paragraph{Proof of \Cref{eq:toprove2-appendix}:} The second equation follows by substituting our setting of $r$ (\Cref{eq:rv}) in $d \cdot v^{r^2/4}$,\begin{align*}
        d \cdot v^{r^2/4} =  2 v^{r + r^2/4} = v^{O((16q)^{m+2}) \log^m v} 
    \end{align*} where the first equality follows from $d = 2 v^r$.
\end{proof}
\end{document}